\documentclass[12pt]{evolang9}

\usepackage{amsmath}
\usepackage{graphicx}
\usepackage{color}

\begin{document}

\newtheorem{theorem}{Theorem}[section]
\newtheorem{lemma}[theorem]{Lemma}
\newtheorem{proposition}[theorem]{Proposition}
\newtheorem{corollary}[theorem]{Corollary}

\newenvironment{proof}[1][Proof]{\begin{trivlist}
\item[\hskip \labelsep {\bfseries #1}]}{\end{trivlist}}
\newenvironment{definition}[1][Definition]{\begin{trivlist}
\item[\hskip \labelsep {\bfseries #1}]}{\end{trivlist}}
\newenvironment{example}[1][Example]{\begin{trivlist}
\item[\hskip \labelsep {\bfseries #1}]}{\end{trivlist}}
\newenvironment{remark}[1][Remark]{\begin{trivlist}
\item[\hskip \labelsep {\bfseries #1}]}{\end{trivlist}}

\newcommand{\qed}{\nobreak \ifvmode \relax \else
      \ifdim\lastskip<1.5em \hskip-\lastskip
      \hskip1.5em plus0em minus0.5em \fi \nobreak
      \vrule height0.75em width0.5em depth0.25em\fi}

\title{The placement of the head that minimizes online memory: a complex systems approach}

\author{Ramon Ferrer-i-Cancho}

\maketitle

\begin{abstract}
It is well known that the length of a syntactic dependency determines its online memory cost. Thus, the problem of the placement of a head and its dependents (complements or modifiers) that minimizes online memory is equivalent to the problem of the minimum linear arrangement of a star tree. However, how that length is translated into cognitive cost is not known. This study shows that the online memory cost is minimized when the head is placed at the center, regardless of the function that transforms length into cost, provided only that this function is strictly monotonically increasing. Online memory defines a quasi-convex adaptive landscape with a single central minimum if the number of elements is odd and two central minima if that number is even. We discuss various aspects of the dynamics of word order of subject (S), verb (V) and object (O) from a complex systems perspective and suggest that word orders tend to evolve by swapping adjacent constituents from an initial or early SOV configuration that is attracted towards a central word order by online memory minimization. We also suggest that the stability of SVO is due to at least two factors, the quasi-convex shape of the adaptive landscape in the online memory dimension and online memory adaptations that avoid regression to SOV. Although OVS is also optimal for placing the verb at the center, its low frequency is explained by its long  distance to the seminal SOV in the permutation space.
\end{abstract}



\section{Introduction}

Word order is a complex phenomenon with different forces pulling in different directions \cite{Langus2010a,Hawkins2004a}. One of these forces is the minimization of the length of a syntactic dependency between a head and its dependents (modifiers or complements). The constraints on word order imposed by this force explain the interpretation of ambiguous sentences \cite{Gibson1998a}, sentence comprehension difficulties \cite{Gibson1998a}, sentence acceptability \cite{Morrill2000a}, word order preferences \cite{Hawkins1994}, the exponential decay of the probability of syntactic dependency lengths \cite{Ferrer2004b}, Greenbergian universals \cite{Ferrer2008e}, the low frequency of crossings between syntactic dependencies \cite{Ferrer2006d} and the tendency of dependencies not to cover the root of a syntactic dependency structure \cite{Ferrer2008e}. Various statistical properties of syntactic dependencies, such as the sublinear scaling of the mean dependency length as a function of sentence length \cite{Ferrer2004b,Ferrer2013c} and below-chance global metrics of dependency lengths \cite{Liu2008a,Gildea2010a} are consistent with a principle of dependency length minimization.
The aim of the present article is to illuminate the complex phenomenon of word order from the perspective of this force, hoping that progress in one dimension helps to solve the whole puzzle.

\begin{tabular}{c}
\\
{\bf Table 1 near here} \\
\\
\end{tabular}

Various sources suggest that there is a tendency for the verb (V) to be placed between the subject (S) and the object (O). Table \ref{word_order_statistics_table} indicates that among all the possible orders of S, V and O, those with V at the center (SVO and OVS) represent $42 \%$ of the languages showing a dominant word order. Although languages with a V-final order are slightly more numerous, i.e. $48\%$, 
\begin{itemize}
\item
the number of languages with V at the center is larger than predicted by the null hypothesis that V can go anywhere \cite{Ferrer2008f}. Under this null hypothesis, the verb has a 1/3 probability of being placed first, second or third;
\item
the most frequent word order by far is SVO \cite{Benz2010a} if frequency is measured in number of speakers and not in number of languages as in Table \ref{word_order_statistics_table}. 
\end{itemize}
Although these statistical results may depend on the choice of the null hypothesis and the subjectivity or measurement error in assigning a dominant word order to a language with traditional methods, they suggest that central verb placement is an attractor of the dynamics of word order evolution and motivate a stronger mathematical approach. However, the high frequency of SVO could be simply an accident of history, e.g., the result of the higher diffusion of a certain culture through imperialism. Stronger support for the hypothesis of central verb placement as an attractor is provided by the direction of word order change. When it occurs, this change has been for the most part from SOV to SVO and, beyond that, from SVO to VSO/VOS with a subsequent reversion to SVO occurring occasionally (interestingly, reversion to SOV occurs only via diffusion) \cite{Gell-Mann2011a}. The idea that SOV evolves towards SVO but evolution in the opposite direction is by far less common is not new \cite{Newmeyer2000,Givon1979a}. 

Here we aim to shed light on the origins of the attraction of V towards the center as a particular case of the problem of arranging a head (e.g. a verb) and its complement(s) (a subject or an object) sequentially. Imagine that there is a head and $n$ dependents (modifiers or complements). For instance,
the English phrase ``a black cat'' is a case of a head, i.e. the noun ``cat'', and two modifiers, i.e. the determiner ``a'' and the adjective
``black''. The head and the modifiers/complements will be referred to as elements (there are
thus $n+1$ elements). Imagine
that the elements are produced sequentially. It will be shown that placing the head at the center minimizes the online memory cost and that dependency length minimization is a consequence of online memory minimization. In particular, this implies that placing the verb between the subject and the object is optimal. These insights challenge the claim that evidence of the selective advantage of SVO is missing \cite{Gell-Mann2011a}. The goal of this article is not to solve the puzzle of the ordering of S, V and O entirely, but rather to shed light on a specific question of word order dynamics: what force could be responsible for an SOV language becoming SVO? 
Similarly, why is the current number of SOV languages historically decreasing \cite{Newmeyer2000}?
Addressing the issue of why SOV is the most frequent word order in languages at present and apparently even more so in ancient times is beyond the scope of this article.

\section{The placement of heads that minimizes the online memory expenditure}

\label{memory_section}

Imagine that the positions of a head and its $n$ dependents in a sequence are specified using
natural numbers from $1$ to $n+1$ and that the position of the head is $l$
(thus $l \in [1,n+1]$). Following these conventions, the phrase ``a black
cat'' has ``a'' at position $1$, ``black'' at position $2$ and ``cat'' at
position $3$, and $n=2$. 
The distance between two elements is defined as the absolute value of the difference between element positions, i.e. the number of intermediate elements plus one \cite{Ferrer2004b}.
For instance, in the example above, "black" and "cat" are at distance one, whereas "a" and "cat" are at distance two.
Similarly, the length of a dependency is defined as the distance between the head and the dependent.
There is a long tradition in linguistics and closely related fields of studying the relationship  
between cognitive cost and the distance between syntactically related items
\cite{Hawkins1994,Gibson1998a,Morrill2000a,Grodner2004a,Liu2008a,Temperley2008a,Ferrer2008e}.
The distance between a head and its dependent can be seen as an estimate of the time that is needed to keep an open or unresolved head-dependent dependency in online memory \cite{Morrill2000a}. Accordingly, $g(d)$ is defined as the online memory cost of a dependency of length $d$ (length is measured in elements). For simplicity, we assume that the online memory cost of a dependency is based only on its length, and thus the only parameter of the online memory cost function is $d$. In particular, this implies the assumption that the cost of a dependency is not influenced by whether the head precedes or follows its dependent. We assume that $g$ is a strictly monotonically increasing function of $d \in [1,n]$. For instance, the identity function ($g(d)=d$) has been considered \cite{Ferrer2004b,Liu2008a}.
We will not work directly on distances but on their implied online memory cost. The total online memory cost of the dependencies between a head placed at position $l$ ($1\leq l \leq n+1$) and its $n$ dependents may be defined as the sum of the cost of dependencies from dependents to the left of the head plus the sum of the cost of dependencies from dependents to its right, i.e.
\begin{eqnarray}
D_l & = & \sum_{i=1}^{l-1} g(|i-l|) + \sum_{i=l+1}^{n+1} g(|i-l|) \notag \\
    & = & \sum_{i=1}^{l-1} g(l-i) + \sum_{i=l+1}^{n+1} g(i-l), \label{raw_total_cost_equation}
\end{eqnarray}
where $|...|$ is the absolute value operator. Eq. \ref{raw_total_cost_equation} can be rewritten as 
\begin{equation}
D_l = \sum_{d=1}^{l-1} g(d) + \sum_{d=1}^{n+1-l} g(d). 
\label{total_cost_equation}
\end{equation}
Although $g(0)=0$ is a reasonable assumption, notice that the definition of $D_l$ in Eq. \ref{total_cost_equation} does not need it because $g(d)$ is always invoked satisfying $1 \leq d \leq n$. Assuming $g(0)=0$ is neither necessary for the arguments below.
When $g(d)=d$, Eq. \ref{total_cost_equation} yields a polynomial of the second degree, i.e. 
\begin{eqnarray}
D_l & = & \sum_{d=1}^{l-1} d + \sum_{d=1}^{n+1-l} d \notag \\
    & = & l(l-1)/2 + (n+2-l)(n+1-l)/2 \notag \\ 
    & = & l^2-(n+2)l + \frac{1}{2}(n+1)(n+2), \label{total_cost_identity_equation}
\end{eqnarray}
after some algebra. For instance, when $n=3$, the head has four possible placements, cf. Fig. \ref{linear_arrangement_figure}. When $n=3$ and $g$ is the identity function, $D_l$ is maximum when the head is placed in the extremes and minimum when it is placed in one of the two middle positions (Eq. \ref{total_cost_identity_equation} gives $D_l=6$ for $l=1$ and $l=4$ and $D_l=4$ otherwise). 

\begin{tabular}{c}
\\
{\bf Fig. 1 near here} \\
\\
\end{tabular}

From a theoretical perspective, the problem of the placement of the head that minimizes $D_l$ is a particular case of the minimum linear arrangement problem for a tree with $g(d) = d$ \cite{Chung1984,Baronchelli2013a}. In our case, our tree is a star tree of $N = n + 1$ vertices, with the head being the hub of the star tree and the edges the syntactic dependencies between the governor and its dependents. Fig. \ref{linear_arrangement_figure} shows various linear arrangements of star trees. Star trees have maximum degree variance and their linear arrangement cannot have crossings \cite{Ferrer2013b}.
Assuming that $g(d) = d$, it has been shown that \cite{Ferrer2013b}
\begin{itemize} 
\item
Star trees reach the maximum sum of dependency lengths ($D_l$) that a non-crossing tree can reach when the hub is placed first ($l=1$) or last ($l=N$), i.e. 
  \begin{equation}
  D_l = \frac{N(N-1)}{2}.
  \end{equation} 
\item
In a star tree, $D_l$ is minimized when the hub is placed at the center. 
If $N$ is odd, then $D_l$ is minimized by $l = (N + 1)/2$ with 
   \begin{equation} 
   D_l = \frac{(N-1)(N+1)}{4}.
   \end{equation} 
If $N$ is even, then $D_l$ is minimized by either $l = N/2$ or $l=N/2 + 1$ with 
   \begin{equation}
   D_l = \frac{N^2}{4}.
   \end{equation}
\end{itemize}
Here we aim to go beyond the assumption of $g(d) = d$, which is common in research on the minimum linear arrangement problem in computer science \cite{Chung1984,Diaz2002} and linguistics from a mathematical \cite{Ferrer2004b,Ferrer2006d,Ferrer2008e,Ferrer2013b} or statistical perspective  \cite{Ferrer2004b,Liu2008a,Temperley2008a,Gildea2010a}. Instead, we show that a similar result holds in the more general case that $g(d)$ is strictly monotonically increasing. 
In her milestone article on optimal linear arrangement of trees, F. Chung proposed to consider $g(d)=d^{\gamma}$ with $\gamma$ being a fixed power. Here we consider the case that $g$ is a strictly monotonically increasing function of $d$, which covers the case $\gamma>0$ on star trees.

A sequence of $n + 1$ elements (with $n \geq 1$) has a single central position at $l^* = \lceil (n+1)/2 \rceil$ if $n$ is even and two central positions, one at $l^*$ and another at $l^* + 1$ if $n$ is odd ($\lceil x \rceil$ is the smallest integer not less than $x$). 
Thus $l^*$ is the only central position if $n$ is odd and the first central position if $n$ is even. If $n$ is even, this reduces to $l^* = n/2 + 1$ while if $n$ is odd, this reduces to $l^* = (n+1)/2$.
Hereafter we assume that $n \geq 2$, since there are no dependents when $n = 0$ and the placement of the head does not matter when $n = 1$ because $D_1 = D_2$. The following theorem states that the online memory cost is maximum when the head is placed at the extremes of the sequence and minimum at the center for any strictly monotonically increasing online memory cost function (some intuition about the result presented below can be obtained assuming $g(d) = d$ \cite{Ferrer2008e,Ferrer2013b}).

\begin{theorem}[Online memory cost of the dependencies]
\label{memory_theorem}
Let $l^*=\lceil (n+1)/2 \rceil$ be a central position of a sequence of $n+1$ elements. If it is assumed that there is a head, $n\geq 2$ dependents, and $g$ is a strictly monotonically increasing function of $d$ in $[1,n]$, then $D_l$, the total cost of the dependencies between the head placed at position
$l \in [1,n+1] \subset \mathbb{N}$ and its $n$ dependents, has
\begin{itemize}
\item
a minimum for $l=l^*$ if $n$ is even;
\item
two minima for $l=l^*$ and $l=l^*+1$ if $n$ is odd;
\item
two identical maxima for $l=1$ and $l=n+1$.
\end{itemize}
Additionally, $D_l$ is
\begin{itemize}
\item
strictly monotonically decreasing for $l\in [1,l^*]$ and strictly monotonically increasing for $l\in [l^*,n+1]$ when $n$ is even;
\item
strictly monotonically decreasing for $l\in [1,l^*]$ and strictly monotonically increasing for $l\in [l^*+1,n+1]$ when $n$ is odd.
\end{itemize}
\end{theorem}

\begin{proof} 
The argument is based on the discrete derivative of $D_l$, i.e. 
\begin{eqnarray}
\Delta_l & = & \frac{D_{l+1} - D_l}{(l+1)-l} \notag \\
         & = & D_{l+1} - D_l.
\label{discrete_derivative}
\end{eqnarray}
Applying the definition of $D_l$ in Eq. \ref{total_cost_equation} to Eq. \ref{discrete_derivative} yields
\begin{eqnarray}
\Delta_l & = & \sum_{d=1}^l g(d) - \sum_{d=1}^{n-l} g(d) - \sum_{d=1}^{l-1} g(d) + \sum_{d=1}^{n + 1 -l} g(d) \notag \\
         & = & g(l) - g(n + 1 - l) \label{delta_equation} 
\end{eqnarray}
for $l \in [1,n]$ (and thus notice that $g(d)$ in Eq. \ref{delta_equation} is only applied to values of $d$ within $[1,n]$, although one could have $n + 1 - l = 0$ or $l = n+1$ {\em a priori}).
Knowing that $g$ is a strictly monotonically increasing function of $d \in [1,n]$, Eq. \ref{delta_equation}, gives
\begin{itemize}
\item
$\Delta_l < 0$ ($D_l$ decreases) iff $l < n + 1 - l$
\item
$\Delta_l > 0$ ($D_l$ increases) iff $l > n + 1 - l$
\item
$\Delta_l = 0$ ($D_l$ is constant) iff $l = n + 1 - l$.
\end{itemize}
Let us define a central position of the sequence as
\begin{equation}
\lambda = (n+1)/2, 
\end{equation}
then 
\begin{enumerate}
\item[a)]
$\Delta_l < 0$ iff $l < \lambda$. Thus, the existence of a natural $l$ within the domain of $\Delta$ (which is $[1,n]$) such that $\Delta_l < 0$ requires that $\lambda$ remains above the smallest possible value of $l$, which is $1$. The condition $\lambda > 1$ yields $n > 1$, which coincides with the assumption of $n \geq 2$ of the theorem. 
\item[b)]
$\Delta_l > 0$ iff $l > \lambda$. Thus, the existence of a natural $l$ within the domain of $\Delta$ such that $\Delta_l > 0$ requires that $\lambda$ remains below  the largest possible value of $l$ within the domain of $\Delta$, which is $n$. The condition yields $n > 1$ again. 
\item[c)]
$\Delta_l = 0$ iff $l = \lambda$. Thus, the existence of a natural $l$ such that $\Delta_l = 0$ requires that $\lambda$ is also a natural number, which only happens if $n$ is odd.
\end{enumerate}
Now the goal is to determine the interval $[l_{min},l_{max}]$ where $D_l$ is strictly monotonically decreasing or strictly monotonically increasing. $l_{min}$ and $l_{max}$ must be natural numbers. Notice that $l^* =  \lceil \lambda \rceil$.
When $n$ is even, $\lambda$ is not natural and $D_l$ is strictly monotonically decreasing for $l\in [1,l^*]$ according to a) and strictly monotonically increasing for $l\in [l^*,n+1]$ according to b). Therefore, $D_l$ has a single minimum at $l=l^*$.
When $n$ is odd, this is straightforward because $\lambda$ is natural, $\lambda = l^*$ and then  
$D_l$ is strictly monotonically decreasing for $l\in [1,l^*]$ according to a) and strictly monotonically increasing for $l\in [l^*+1,n+1]$ according to b).
Therefore, recalling c), i.e. $\Delta_l = 0$, it follows that $D_l$ has two minima, one for $l=l^*$ and the other for $l=l^*+1$.

Concerning the maxima, notice that $D_l$ is a symmetric function of $l$, i.e. $D_l=D_{n+2-l}$, by its definition in Eq. \ref{total_cost_equation}.
According to the shrinking and growth behavior of $D_l$ described above, $D_l$ has a maximum for $l=1$ and another for $l=n+1$. These two maxima are identical due to the symmetry of $D_l$. \qed      
\end{proof}

Theorem \ref{memory_theorem} indicates that $-D_l$ is a unimodal function if $n$ is even \cite[p. 63]{Avriel1988a} but almost unimodal if $n$ is odd, as the two modes are adjacent in that case. It will be shown next that $D_l$ is a quasi-convex function. 
Quasi-convexity is a generalization of convexity \cite{Avriel1988a,Greenberg1971a}. Quasi-convex functions can be optimized within a reasonable computation cost \cite{Kiwiel2001a}.
\begin{corollary}[Quasi-convexity of the online memory cost]
\label{quasiconvexity_theorem}
If there are $n \geq 2$ dependents and $g$ is a strictly monotonically increasing function of $d$, then $D_l$ is quasi-convex within $[1,n+1]$, i.e.  for any $l_1$, $l_2$, and $l_3$ such that $1 \leq l_1 \leq l_2 \leq l_3 \leq n + 1$ one has that 
\begin{equation}
D_{l_2} \leq max(D_{l_1},D_{l_3}). \label{quasiconvexity_equation}
\end{equation}
\end{corollary}
\begin{proof}
The condition defined in Eq. \ref{quasiconvexity_equation} is equivalent to (a) $D_{l_2} \leq D_{l_1}$ or (b) $D_{l_2} \leq D_{l_3}$. If $l_2 \leq l^*$, Theorem \ref{memory_theorem} gives (a) while if $l_2 \geq l^*$, Theorem \ref{memory_theorem} gives (b).   
\qed
\end{proof}



\section{Discussion}

It has been shown that placing a head at the center minimizes the online memory expenditure. If the number of elements is odd, there is a single minimum, while, if the number of elements is even, there are two central minima. There are two novelties in our analysis. First, our approach abstracts away from the particular form of the function that translates the distance between a head and its complements into a cognitive cost. The point is subtle: even if one considers that online memory cost originates from the time that is needed to keep an open or unresolved dependency between a head and a dependent in online memory \cite{Morrill2000a}, it is still not known how this time translates into an energy cost for the brain, and if this final translation depends on the language. Second, the form of the adaptive landscape \cite{Wright1932a} defined by online memory cost has been unraveled: 
the landscape is quasi-convex, as illustrated by the example in Fig. \ref{energy_landscape_figure}. In this figure, it has been assumed that the length of a dependency and its cost are the same. Interestingly, Theorem \ref{memory_theorem} indicates that Fig.  \ref{energy_landscape_figure} is representative of the form of the landscape in spite of illustrating only a particular online memory cost function, because the landscape remains quasi-convex for any strictly monotonically increasing positive function that maps length onto cost. Therefore, the landscape depicted in Fig. \ref{energy_landscape_figure}
would be the kind of landscape in which word order dynamics operates on the online-memory dimension. 
The dynamics of the ordering of S, V and O would have at least three fundamental ingredients:  
\begin{enumerate}
\item
SOV as an initial or early word order configuration \cite{Goldin-Meadow2008a,Sandler2005a,Newmeyer2000},
\item
An online memory landscape that would drag the verb from the final position found in SOV towards the center, eventually leading to the SVO order,
\item
Adaptations to increase the stability of a given order (e.g., adaptations in SVO that prevent regression to SOV).
\end{enumerate}

\begin{tabular}{c}
\\
{\bf Fig. 2 near here} \\
\\
\end{tabular}

Theorem \ref{memory_theorem} sheds new light on the transition from SOV to SVO \cite{Newmeyer2000,Givon1979a}, which is a transition from maximum to minimum online memory cost. 
Future work should consider the inclusion of more dimensions, such as other sources of cognitive cost beyond dependency length. An important issue for further research is to determine whether verb-last or verb-central are stable or unstable attractors of word order evolution and what role factors other than dependency length play in determining such stability or instability. 

The return to SVO witnessed occasionally during word order evolution and the difficulty of reversion to SOV except via diffusion \cite{Gell-Mann2011a} suggests some degree of stability for central verbs. The shape of the adaptive landscape of online memory adds new theoretical support for the stability of central verbs: a language that tends to place the verb at the center will receive an increasing penalty as the verb is displaced to the beginning or end due to the quasi-convexity of that shape.

If there is a force that explains why SOV is initially selected, why do SVO languages not return to SOV easily because of that force?
A possible explanation is not only the attraction for the verb at the center due to online memory minimization,
 but also the possibility that evolutionary successful SVO orders may have incorporated further adaptations that hinder regression to SOV (see Appendix). Indeed, stronger theoretical support for the stability of SVO comes from mathematical word order theory, which addresses the issue of the optimal placement of dependents relative to their heads within S, V and O according to online memory minimization \cite{Ferrer2008e}. 
For a language that tends to put the verb last, e.g., an SOV language, the optimal solution in terms of online memory minimization for top-level dependencies is placing dependents before the nominal head whereas by symmetry, for a language that tends to put the verb first, the optimal solution is placing the dependents after the nominal head (Appendix). Interestingly, in the case of verb-central languages, whether dependents are placed to the left or to the right, is almost irrelevant in terms of online memory minimization of top-level dependencies (Appendix). From an evolutionary perspective, a change in the relative placement of dependents of nominal heads is practically neutral for SVO but not for SOV. If one considers that orders are competing and adapting to survive \cite{Ferrer2008e,Gong2009a}, it is not surprising that stable SVO languages put adjectives after the head \cite{Ferrer2008e}, as SOV prefers the opposite, namely that dependents precede the nominal head. This preference for noun-adjective order in SVO languages may be viewed as an obstacle to regression to SOV order. Besides word order, there might be other factors impeding a return to SOV. One possibility is case marking, which facilitates the learning of SOV structures \cite{Lupyan2002a}. Thus, regression to SOV could be harder from SVO languages lacking case marking.

We hypothesize that online memory minimization is a universal principle of language. It is important not to be side-tracked by seemingly contradictory evidence. The fact that a certain language does not show SVO as dominant word order does not mean that language does not suffer pressure for online memory minimization: 
\begin{itemize}
\item
It is well known that SVO is an alternative word order in many languages where
SVO is not the dominant word order \cite{Greenberg1963a}. This has a simple explanation. Placing the verb at the center is optimal in terms of online memory minimization; placing it somewhere else is not (Theorem \ref{memory_theorem}). Therefore, a language that does not have a dominant
word order with the verb at the center can compensate the cost of that non-optimal verb placement by adopting an alternative word order that is optimal such as SVO.
\item
The abundance of verb-final orders does not contradict the principle of online memory cost minimization. First, the diversity of orders of S, V and O \cite{Dryer2011a} could result from various principles acting simultaneously \cite{Langus2010a,Hawkins2004a}. Second, the relative position of adjectives and verbal auxiliaries in verb-final orders can be explained in terms of online memory cost minimization \cite{Ferrer2008e}. Thus, the fact that a language has SOV as the dominant order does not mean that online memory cost minimization is inactive. 
\end{itemize}
It is tempting to think that, if the hypothesis of attraction of the verb towards the center is correct, then OVS, the other ordering with V at the center, should have high frequency and the transition from SOV to OVS should be as frequent as the transition from SOV to SVO. However, the failure of these predictions on OVS can be understood easily with the help of the permutation space, which was originally introduced to explain the low frequency of OVS with regard to SVO \cite{Ferrer2008e}. If one assumes that word order evolves by swapping consecutive elements, the evolution from an initial or early word order SOV to SVO requires only one step: exchanging the positions of O and V. In contrast, the transition from SOV to OVS needs two steps: (1) exchanging the positions of S and O to obtain OSV and (2) exchanging S and V to finally obtain OVS (Fig. \ref{permutation_ring_figure} (a)).
Therefore, the verb-central word order that can be reached sooner from initial or early SOV is SVO. SVO is the easiest way of minimizing online memory from SOV by swapping adjacent constituents. While it has been argued that the disproportion between the abundance of SVO (488 languages) and OVS (11 languages) could be due to an arbitrary break of the symmetry between orders placing the verb at the center \cite{Ferrer2008e}, we argue here that this is more likely to be caused by the proximity of SOV, an attractor of word order evolution at early stages.
For simplicity, we have assumed that word order evolution proceeds by exchanging consecutive elements, but it could be argued that the exchange of distant elements has an important role in word order evolution. However, the latter is expected to be cognitively more expensive and therefore less likely (consider all the arguments supporting online memory minimization reviewed above). 

\begin{tabular}{c}
\\
{\bf Fig. 3 near here} \\
\\
\end{tabular}

The path of the evolution of word order, i.e. SOV, SVO and VSO/VOS \cite{Gell-Mann2011a} is consistent with a traversal of the permutation space characterized by the swapping of adjacent constituents. It could be argued that once SVO is reached, the emergence of its reverse, OVS should be easier. However, notice that OVS is the farthest order in the permutation ring depicted in Fig. \ref{permutation_ring_figure} (a): at least three permutations of adjacent constituents are needed to reach OVS from SVO. Interestingly, the frequency of languages showing a certain dominant word order X is perfectly correlated with the number of swaps needed to reach X from SOV (Fig. \ref{permutation_ring_figure} (b)): the number of languages always decreases as the number of swaps (of adjacent constituents) increases in a clockwise sense in the permutation ring of Fig. \ref{permutation_ring_figure} (a). Thus the Spearman rank correlation \cite{Spearman1904a} is $\rho = -1$ and the p-value of a two-sided test \cite{Conover1999a} is $2/6! \approx 0.0028$, as all the frequencies of dominant orders and the number of swaps do not have repeated values \cite{Ferrer2012a}.
We have used a Spearman rank correlation for its capacity to capture non-linear correlations \cite{Zou2003a}.
The power of the test would have dropped significantly if a Pearson correlation test had been used: 
with a Pearson correlation test, we would get $r = -0.89$ with a p-value of $0.016$. The reduced power of the Pearson correlation is consistent with the non-linear appearance of the number of languages as a function of the number of swaps (Fig. \ref{languages_versus_distance_figure}). The Pearson correlation is not able to capture the perfect non-linear correlation ($r > -1$  indicates a weaker correlation, compared to $\rho = -1$), and gives a p-value that is about 6 times larger than that of the Spearman correlation.

\begin{tabular}{c}
\\
{\bf Fig. 4 near here} \\
\\
\end{tabular}

Besides the shape of the permutation space and the clockwise sense, other factors could explain why OVS is rarely reached, such as a preference for initial S that is suggested by the very high frequency of SOV and SVO together (initial S is found in $88\%$ of languages showing a dominant word order \cite{Dryer2011a}). Another possible factor might be the preference for SO over OS \cite{Cysouw2008a}. However, this explanation would imply adopting an approach to word order that contrasts with the conceptual design of our complex systems approach. Assuming a combination of a preference for SO over OS and a preference for SV over VS (which implies a preference for S first) to explain the frequency of the dominant ordering of S, O and V in languages \cite{Cysouw2008a} is a reductionistic view: the order of the whole is believed to reduce to the relative ordering of the parts. This is problematic because the optimal linear arrangement of the pair S and V does not need to be the same as that of the triple formed by the S-V pair and O. For instance, in terms of online memory minimization of the pair S and V, the ordering (SV or VS) is irrelevant, but if O is also present then V must be placed at the center \cite{Ferrer2008e}. How our emergentist approach and the traditional reductionistic approach could be integrated with each other is the subject of future research.
 

Here we have considered the problem of the optimal linear arrangement of dependents with regard to their head, focusing on the particular case of the verb, subject and object. As our argument is abstract, one expects that it is also valid for other heads and their dependents. For instance, a challenge is that SVO languages tend to have the adjective after the noun \cite{Ferrer2008e}. 
The fact that no language consistently splits its noun phrases around a central, nominal pivot, with half of the modifiers to the left and half to the right, might be seen as evidence that online memory minimization has a very restricted scope. However, online memory minimization is still a fundamental principle even for nominal heads. The point is that the optimal placement of modifiers around a noun in terms of online memory minimization is (a) not independent from the dominant placement of the verb governing the noun in a word order $X$ (as we have discussed above), and (b) not necessarily independent from a competing word order $Y$. Consider that $X$ is SVO. Then the permutation ring in Fig. \ref{permutation_ring_figure} (a) indicates that $Y$ could be SOV, as it is easy to move from SVO to SOV. From the perspective of online memory minimization, SVO may put half of the modifiers to the left and half of the modifiers to the right of the noun; but as the placement of modifiers before the nomimal head is optimal in SOV, we have seen that SVO languages can maximize the cost for SOV by placing modifiers on the opposite side \cite{Ferrer2008e}. A non-reductionistic approach to word order needs to take into account both (a) the interaction between the placement of heads of different levels, and (b) the interaction between competing word orders at the same level.

A basin of attraction consists of an attractor and all the trajectories leading to this attractor \cite{Wuensche2000a,Riley2012a}.
Figure \ref{energy_landscape_figure} defines the trajectories and their impact on online memory cost when moving the head one position forward or backward, and thus defines the basin of attraction of word order evolution on the single dimension of online memory cost. 
Fig. \ref{permutation_ring_figure} (a) defines the trajectories that can be followed {\em a priori} by swapping consecutive constituents till SVO is reached. 
Looking at the neighbors of SOV in the permutation ring, one notices that SVO (488 languages) is much more frequent than OSV (4 languages). This suggests that the bulk of the trajectories of dominant word order evolution might be more accurately described by a directed graph version of Fig. \ref{permutation_ring_figure} (a) where the link direction is clockwise. Fig. \ref{permutation_ring_figure} (b) shows (from left to right) the linear undirected tree that results from a traversal of the ring of Fig. \ref{permutation_ring_figure} (a) starting from SOV, moving clockwise and stopping when all word orders have been visited. The low frequency of OSV languages and the correlation between word order frequency and distance to SOV provides further support for the idea that word order evolution is essentially unidirectional \cite{Gell-Mann2011a}, a feature that is reminiscent of the unidirectionality of grammaticalization processes \cite{Traugott1991}.
However, a directed graph with a ring backbone is still an incomplete draft of the possible trajectories under the influence of -- at least -- online memory minimization. Contingency (initial or early preference for SOV) would favor SVO over OVS. But SOV itself could be an attractor whose origin, from a dynamical and mathematical point of view, is still not well-understood.

In order to develop a deep understanding of language and word order in particular, it is important not to mistake a principle for an explanation.
The fact that online memory minimization does not explain the abundance of SOV order does not mean that it has nothing to do with why and how it appears, or why and how it is maintained. If a language leaves SOV, it is well known that SVO is a dominant direction of change; and although SVO can be abandoned towards VSO/VOS, reversion to SVO occurs occasionally. Furthermore, reversion to SOV is not expected except via diffusion, which is a secondary process in word order evolution \cite{Gell-Mann2011a}. It is perhaps here that the power of our mathematical results is revealed. Two languages could even have a different function for transforming the length of a dependency into a cognitive cost, but, provided that the transformation is monotonically increasing, both will at least be attracted towards central head placement. Furthermore, the form of the adaptive landscape would reinforce the attraction towards a central placement of the verb (Fig. \ref{energy_landscape_figure}). 
In this view, online memory minimization is neither regarded as the only principle nor as the most important one. Language design is a multiconstraint engineering problem \cite{Evans2009a} and, therefore, online memory minimization would be just one of the constraints to meet during word order evolution. Our analysis has simply focused on one dimension of the problem. 

A look through the eye of physics can help clarify our notion of principle. The force of gravity explains why objects fall, but when a rocket flies in the opposite direction of that force, one would not say that its movement constitutes an exception to 
gravity. The force is still acting and is involved in explaining, for instance, the amount of fuel
that is needed to fly in the opposite direction. Just like the falling of an object is a
manifestation of the force of gravity, central head placement or the movement of a head towards the center (perhaps not reaching the optimal center) are manifestations of a principle of online memory minimization. And, just like the force of gravity is still acting on a rocket flying in the opposite direction, one should
not conclude prematurely that online memory minimization is not acting upon a language when the head is not placed in the middle, i.e. the optimal position according to online memory minimization. Online memory minimization still determines the relative placement of adjectives and verbal auxiliaries in SOV languages \cite{Ferrer2008e}. On a larger scale, 
although Earth and Venus are very different planets, the force of gravity is valid in both
(indeed universal) and physicists only care about the variation in its magnitude. Similarly,
all languages can have online memory minimization in common, no matter how large the genetic, historical, typological or other differences are among the languages or among their speakers. We are challenging the claim that central verb placement does not confer any selective advantage to a language that adopts it \cite{Gell-Mann2011a}. In sum, we believe that a more physics-oriented point of view is needed for progress in our understanding of word order evolution. This viewpoint reveals how online memory minimization could underlie the transition from SOV to SVO \cite{Newmeyer2000,Givon1979a}.
However, understanding the tendency of subjects to precede objects \cite{Cysouw2008a} and other principles that can determine a different verb placements such as the verb-final placement found in the very frequent SOV languages, remain issues for future research. 

\section*{Acknowledgments}

An early version of the main of results of this article was presented in the Kickoff Meeting "Linguistic Networks" which was held in Bielefeld University (Germany) in June 5, 2009. We thank the participants, specially G. Heyer and A. Mehler for valuable discussions and S. Whichmann and B. Elvev{\aa}g for many indications to improve the current article. The version published in this journal has benefited enormously from the the deep and extensive comments of Phillip Alday and also the extremelly careful proofreading of Silke Lambert.
We are also grateful to S. Caldeira, M. Christiansen, F. Jaeger, E. Gibson, R. Levy, H. Liu, N. Moloney, G. Morrill, S. Piantadosi, F. Reali and O. Valent\'{i}n for helpful comments at various stages. This work was supported by the grant {\em Iniciaci\'o i reincorporaci\'o a la recerca} from the Universitat Polit\`ecnica de Catalunya, and the grants BASMATI (TIN2011-27479-C04-03) and OpenMT-2 (TIN2009-14675-C03) from the Spanish Ministry of Science and Innovation.

\bibliographystyle{apacite}
\bibliography{../biblio/complex,../biblio/rferrericancho,../biblio/cl,../biblio/ling,../biblio/cs,../biblio/maths} 

\section*{Appendix: optimal placement of dependents of nominal heads}


Assuming $g(d) = d$, a mathematical theory for the ordering of three top-level constituents, subject (S), object (O) and verb (V), was developed previously \cite{Ferrer2008e}. 
$|x|$ is defined as the length in words of constituent $x$ ($x \in \{$S, V, O$\}$). We assume that the constituents are not empty ($|O|,|S|,|V| \geq 1$) and their lengths are constant (word order variations do not alter $|O|$,$|S|$ or $|V|$).
$\delta_{x \sim x'}^y$ is defined as the length of the dependency between the heads of constituents $x$ and $x'$ in word order $y$ ($y \in \{SOV, SVO,...\}$).
$\delta^y = \delta_{V \sim S}^y + \delta_{V \sim O}^y$ is thus the sum of the length of the dependencies (in words) between the head of the verb and the subordinated heads of the subject and the object for a word order $y$. The total cost of a sentence consisting of $S$, $V$ and $O$ following an order $y$ is 
\begin{equation}
\Omega^{y} = \omega_S^{y} + \omega_V^{y} + \omega_O^{y} + \delta^{y},
\end{equation}
where $\omega_x^y$ is the total sum of the internal dependency lengths of constituent $x$ in word order $y$. 

$L_x^y$ and $R_x^y$ are defined, respectively, as the number of words to the left and to the right of the head word of constituent $x$ (e.g., $x \in \{$S, V, O$\}$) in word order $y$, and thus $|x| = L_x^y + 1 + R_x^y$.
Assuming continuity, the following holds \cite{Ferrer2008e}:
\begin{equation}
\delta^{SOV} = 2L_V^{SOV} + 2R_O^{SOV} + L_O^{SOV} + R_S^{SOV} + 3
\label{sum_of_dependency_lengths_SOV_equation}
\end{equation}
and
\begin{equation}
\delta^{SVO} = R_S^{SVO} + |V| + L_O^{SVO} + 1.
\label{sum_of_dependency_lengths_SVO_equation}
\end{equation}
For simplicity, the original mathematical theory \cite{Ferrer2008e} focused on the problem of the optimal placement of dependents of nominal heads to minimize the sum of lengths of dependencies defined by the top-level constituents (external dependencies), i.e. $\delta^y$, thus neglecting the cost of dependencies formed by words within a given constituent, i.e. $\omega_S^{y}$, $\omega_V^{y}$ and $\omega_O^{y}$. As will be shown next, this is equivalent to assuming conservation for the lengths implied by other dependencies (internal dependencies) when varying the relative ordering of those top-level constituents.

$\delta^{y,left}$ and $\delta^{y,right}$ are defined as the value of $\delta^y$ implied by placing all the dependents of the nominal head before or after the noun, respectively.
$\omega_x^{{y},left}$ and $\omega_x^{{y},right}$ are defined as the total sum of the internal dependencies of constituent $x$ in word order $y$ when dependents of nominal heads precede or follow the head, respectively. $\Omega^{y,left}$ and $\Omega^{y,right}$ are defined similarly for the total sum. Thus we have 
\begin{eqnarray}
\Omega^{y,left} = \omega_S^{y,left} + \omega_V^{y,left} + \omega_O^{y,left} + \delta^{y,left}, \\
\Omega^{y,right} = \omega_S^{y,right} + \omega_V^{y,right} + \omega_O^{y,right} + \delta^{y,right}.
\end{eqnarray}
The conservation of the total sum of internal dependency lengths when varying the relative ordering of the dependents of nominal heads means that
\begin{equation}
\omega_S^{y,left} + \omega_V^{y,left} + \omega_O^{y,left} = \omega_S^{y,right} + \omega_V^{y,right} + \omega_O^{y,right} \label{conservation1_equation}
\end{equation}
or equivalently 
\begin{equation}
\Omega^{y,left} - \Omega^{y,right} = \delta^{y,left} - \delta^{y,right}, \label{conservation2_equation}
\end{equation}
i.e. conservation means that the difference in total cost between different relative placements depends only on the sum of dependency lengths of top-level constituents. 
Investigating the extent to which this conservation is valid for human language and the consequences of violations of such conservation for our theoretical arguments is left for future research. 

Next, the sum of dependency lengths as a function of the relative placement of dependents of nominal heads for a given word order will be investigated. Notice that the condition $\Omega^{y,left} \leq  \Omega^{y,right}$ is equivalent to $\delta^{y,left} \leq \delta^{y,right}$ under conservation (Eq. \ref{conservation2_equation}).
Let us consider SOV first. 
For placements of dependents before their nominal head, one has $R_O^{SOV, left} = 0$, $L_O^{SOV, left} = |O| - 1$ and $R_S^{SOV,left} = 0$, which transforms Eq. \ref{sum_of_dependency_lengths_SOV_equation} into 
\begin{equation}
\delta^{SOV, left} = 2L_V^{SOV,left} + |O| + 2.
\label{sum_of_dependency_lengths_SOV_before_equation}
\end{equation}
For dependents following their nominal head, one has $R_O^{SOV,right} = |O|-1$, $L_O^{SOV,right} = 0$ and $R_S^{SOV,right} = |S|-1$, which transforms Eq. \ref{sum_of_dependency_lengths_SOV_equation} into
\begin{equation}
\delta^{SOV, right} = 2L_V^{SOV,right} + 2|O| + |S|.
\label{sum_of_dependency_lengths_SOV_after_equation}
\end{equation}
Assuming $L_V^{SOV,left} = L_V^{SOV,right}$ (the relative ordering of dependents of nominal heads should not alter the relative ordering of dependents of verbal heads), Eqs. \ref{sum_of_dependency_lengths_SOV_before_equation} and \ref{sum_of_dependency_lengths_SOV_after_equation} transform the condition $\delta^{SOV,left} \leq \delta^{SOV,right}$ into 
\begin{equation}
|S| + |O| \geq 2
\label{condition_for_relative_placement_of_dependents_in_SOV_equation}
\end{equation}
with equality if and only if $|S| = |O| = 1$, as $|S|, |O| \geq 1$. Therefore, placing the dependents before their nominal head is always advantageous in terms of minimizing top-level dependency lengths for SOV. The case $|S| = |O| = 1$ does not pose a problem as it means that the nominal heads do not have dependents; but notice that $|S| > 1$ or $|O| > 1$ suffices for the optimality of placing the dependents of nominal heads before the heads of S and O.

Let us consider SVO now. For placements of dependents before their nominal head, one has $R_S^{SVO,left} = 0$ and $L_O^{SVO,left} = |O| - 1$, which transforms Eq. \ref{sum_of_dependency_lengths_SVO_equation} into 
\begin{equation}
\delta^{SVO, left} = |V| + |O|.
\label{sum_of_dependency_lengths_SVO_before_equation}
\end{equation}
For placements after the nominal head, one has $R_S^{SVO,right} = |S|-1$ and $L_O^{SVO,right} = 0$, which transforms Eq. \ref{sum_of_dependency_lengths_SVO_equation} into 
\begin{equation}
\delta^{SVO, right} = |V| + |S|.
\label{sum_of_dependency_lengths_SVO_after_equation}
\end{equation}
Eqs. \ref{sum_of_dependency_lengths_SVO_before_equation} and \ref{sum_of_dependency_lengths_SVO_after_equation} transform the condition $\delta^{SVO,left} \leq \delta^{SVO, right}$ into 
\begin{equation}
|O| \leq |S|,
\end{equation}
with equality if and only if $|O| = |S|$.
This result indicates that, in contrast to SOV, the placement of modifiers with regard to the nominal head is practically irrelevant for SVO. If one assumes that the object and the subject have the same length ($|O|=|S|$) then the total cost of SVO does not depend on the relative ordering of dependents of nominal heads ($\delta^{SVO,left} = \delta^{SVO, right}$). The requirement $|O|=|S|$ might seem too strict, but we have employed a simple mathematical approach. A more powerful mathematical argument is not based on the values of 
$\delta^{SVO,left}$ and $\delta^{SVO, right}$ but on their expectations, namely $E[\delta^{SVO,left}]$ and $E[\delta^{SVO, right}]$. In that case, following similar arguments, one arrives at the conclusion that a tie between relative placements in SVO occurs when $E[\delta^{SVO,left}] = E[\delta^{SVO, right}]$, which is equivalent to $E[|O|]=E[|S|]$. This implies that 
identical lengths, on average, of O and S suffice for the relative ordering of dependents of nominal heads to be irrelevant. An approach based on expectations will be elaborated with more detail in the future. 

Next, a mathematical argument on the cost of regression from SVO to SOV depending on the relative placement of dependents of nominal heads will be developed. 
By definition, the cost of an SVO order that places those dependents before their nominal head is 
\begin{equation} 
\Omega_{SVO} = \omega_S^{SVO,left} + \omega_V^{SVO,left} + \omega_O^{SVO,left} + \delta^{SVO,left}
\label{SVO_cost_left_equation}
\end{equation}
while that of an SVO order placing them after the head is 
\begin{equation}
\Omega_{SVO} = \omega_S^{SVO,right} + \omega_V^{SVO,right} + \omega_O^{SVO,right} + \delta^{SVO,right}.
\end{equation}
Imagine that one of those SVO orders is transformed into SOV simply by reordering the top-level constituents (keeping their internal organization constant so that $\omega_S^{SVO,left}$, $\omega_V^{SVO,left}$, $\omega_O^{SVO,left}$ are not altered; this could be simply due to least effort). Eq. \ref{SVO_cost_left_equation} implies that the cost of the SOV order placing those dependents before their head is
\begin{equation}
\Omega_{SOV}' = \omega_S^{SVO,left} + \omega_V^{SVO,left} + \omega_O^{SVO,left} + \delta^{SOV,left}
\end{equation}
while that of the SOV order placing them after the head is 
\begin{equation}
\Omega_{SOV}'' = \omega_S^{SVO,right} + \omega_V^{SVO,right} + \omega_O^{SVO,right} + \delta^{SOV,right}.
\end{equation}

Thus, the condition $\Omega_{SOV}' < \Omega_{SOV}''$ reduces to $\delta^{SOV,left} < \delta^{SOV,right}$, which has been proven above to require only that $S$, $O$ or both have their own dependents ($|S|>1$ or $|O|>1$; recall Eq. \ref{condition_for_relative_placement_of_dependents_in_SOV_equation}). Regression to SOV from SVO is therefore more expensive from the perspective of SOV when dependents of nominal heads follow their head in SVO.

One could argue that we are oversimplifying the problem of regression from SVO to SOV when considering the case of a reordering of constituents but not an internal reordering of the constituents. However, internal reordering would be less likely because it would increase the cost of regression to SOV (speakers and hearers may find the larger number of word order changes harder to produce or understand). Generalizing the analysis of this appendix to include the case $g(d) \neq d$ is left for future research.

\pagebreak

\begin{table}
\tablecaption{The frequency of the placement of the verb in the ordering of subject (S), verb (V) and object (O) in languages showing a dominant word order. There are three possible placements: 1 for verb-initial orderings (VSO and VOS), 2 for central verb placements (SVO and its reverse) and 3 for verb-final orderings (SOV and OSV). Absolute frequencies are borrowed from \protect \citeA{Dryer2011a}.}
{\footnotesize
\begin{tabular}{@{}lrr}
Verb placement & Frequency (in languages) & Percentage \\
\hline
1 &  120 & 10.1 \\
2 &  499 & 42.0 \\
3 &  569 & 47.9 \\
Total & 1188 & \\
\end{tabular}
\label{word_order_statistics_table}
}
\end{table}

\pagebreak

\begin{figure}[!t]
\centering
\includegraphics[scale = 0.7]{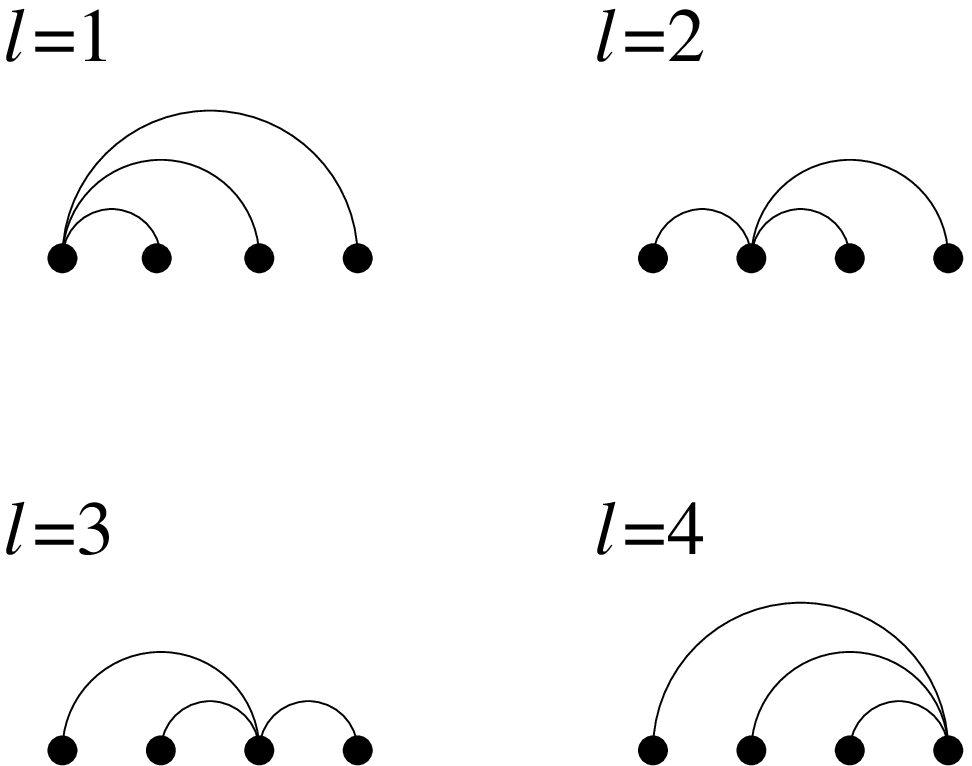}
\caption{All four possible placements of a head with $n=3$ dependents. Black filled circles are elements, and edges indicate syntactic dependencies. $l$ indicates the position of the head in the sequence of elements. It can be seen that the linear arrangements of the top row are symmetric to those of the bottom row.}
\label{linear_arrangement_figure}
\end{figure}

\pagebreak

\begin{figure}[!t]
\centering
\includegraphics[scale = 0.6]{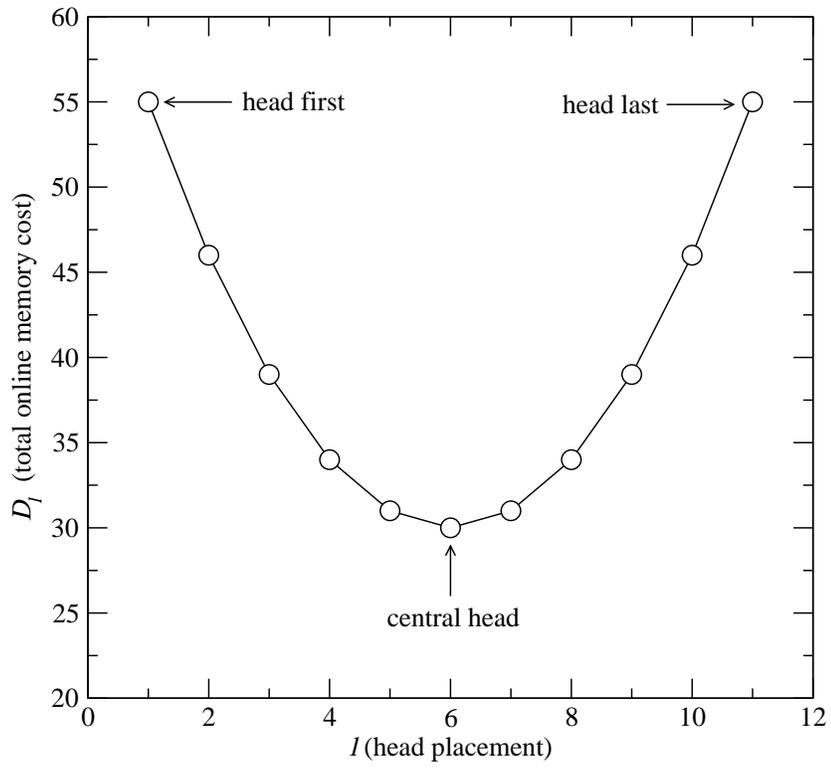}
\vspace{1cm} 
\caption{$D_l$, the total online memory cost of placing a head and $n=10$ dependents as a function of $l$, the position of the head. Eq. \ref{total_cost_identity_equation} is used to compute $D_l$. The placement of the head in first or last position yields maximum cost while the cost is minimum with the head at the center. }
\label{energy_landscape_figure}
\end{figure}

\pagebreak

\begin{figure}[!t]
\centering
\includegraphics[scale = 0.6]{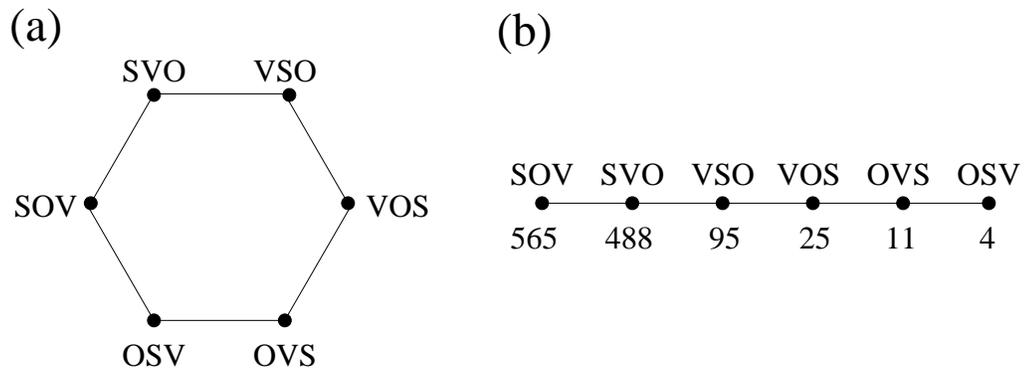}
\vspace{1cm} 
\caption{(a) Network of the possible trajectories between the six possible orderings of S, V and O. Two orders are connected if one can become the other by swapping a pair of adjacent constituents (adapted from \protect \citeA{Ferrer2008e}). The network shows a ring structure. (b) The linear network of trajectories from an initial SOV order to other orders that results from swapping adjacent constituents, following the permutation ring in a clockwise direction. Numbers indicate the number of languages having a certain word order as dominant (numbers borrowed from \protect \citeA{Dryer2011a}). }
\label{permutation_ring_figure}
\end{figure}

\pagebreak

\begin{figure}[!t]
\centering
\includegraphics[scale = 0.6]{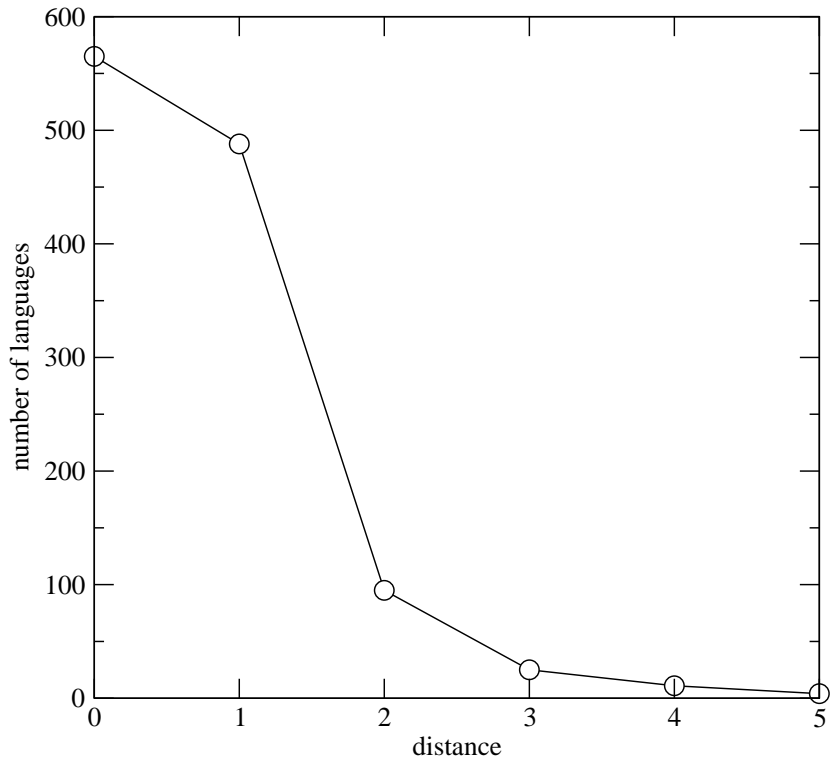}
\vspace{1cm} 
\caption{The number of languages having a certain dominant ordering of S, V and O as a function of their distance (in edges) to SOV, moving clockwise in the permutation ring of Fig. \ref{permutation_ring_figure}.}
\label{languages_versus_distance_figure}
\end{figure}

\end{document}